\pdfoutput=1
\documentclass{article}

\usepackage[nonatbib,preprint]{nips_2018}

\usepackage[utf8]{inputenc} % allow utf-8 input
\usepackage[T1]{fontenc}    % use 8-bit T1 fonts
\usepackage{hyperref}       % hyperlinks
\usepackage{url}            % simple URL typesetting
\usepackage{booktabs}       % professional-quality tables
\usepackage{amsfonts}       % blackboard math symbols
\usepackage{amsmath}
\usepackage{amsthm}
\usepackage{bm}
\usepackage{multicol}
\usepackage{nicefrac}       % compact symbols for 1/2, etc.
\usepackage{microtype}      % microtypography
\usepackage{xspace}      % microtypography
\usepackage{graphicx}
\usepackage{placeins}
\usepackage{hyperref}
\usepackage[capitalize]{cleveref}
\usepackage{caption}
\usepackage{subcaption}
\usepackage{bm}
\usepackage{xcolor}
\usepackage{amsthm}
\usepackage{amsfonts}
\usepackage{stmaryrd}
\usepackage{booktabs} % for professional tables
\usepackage{paralist}
\usepackage{xspace}

\newcommand{\norm}[1]{\left\lVert#1\right\rVert}
\newcommand{\abs}[1]{\left|#1\right|}
\newcommand{\diag}[1]{\text{diag}\left(#1\right)}
\newcommand{\intint}[1]{\left\llbracket#1\right\rrbracket}
\newtheorem{proposition}{Proposition}[section]

\newcommand{\shrink}{{\bf Smallify}\xspace}
\newcommand{\swl}{{\bf SwitchLayer}\xspace}
\newcommand{\swls}{{\bf SwitchLayers}\xspace}

\title{Smallify: Learning Network Size while Training}

\author{
  Guillaume Leclerc, Manasi Vartak, Raul Castro Fernandez, Tim Kraska, Samuel Madden \\
  Computer Science and Artificial Intelligence Laboratory\\
  Massachusetts Institute of Technology\\
  Cambridge, MA 02145 \\
  \texttt\{leclerc, mvartak, raulfc, kraska, madden\}@mit.edu}

\begin{document}

\maketitle

\begin{abstract}
As neural networks become widely deployed in different applications and on
different hardware, it has become increasingly important to optimize inference
time and model size along with model accuracy.  Most current techniques
optimize model size, model accuracy and inference time in different stages,
resulting in suboptimal results and computational inefficiency.  In this work,
we propose a new technique called \shrink that optimizes all three of these
metrics at the same time.  Specifically we present a new method to
simultaneously optimize network size and model performance by neuron-level
pruning during training.  Neuron-level pruning not only produces much smaller
networks  but also produces dense weight matrices
that are amenable to efficient inference.  By applying our technique to
convolutional as well as fully connected models, we show that \shrink can
reduce network size by 35X with a 6X improvement in inference time with similar
accuracy as models found by traditional training techniques. 
\end{abstract}

%!TEX root=paper.tex
\pdfoutput=1

\section{Introduction}

Neural networks are used in an increasingly wide variety of
applications on a diverse set of hardware architectures, ranging from laptops to
phones to embedded sensors. This wide variety of deployment settings means
that inference time and model size are becoming as important as  prediction
accuracy when assessing model quality.  
However, currently these three dimensions, prediction accuracy, inference time, and model size, are optimized independently, often with sub-optimal results. 
%However, these three dimensions of performance (prediction accuracy, inference time and model size) are, in current ML deployments, optimized independently, often with sub-optimal results.

Our approach to optimize the three dimensions also stands in contrast to existing techniques, which can be categorized into two general approaches:
%Of course, the problem of finding faster and smaller networks is not new.
%Existing techniques to make pre-trained neural networks smaller can be categorized in two approaches: 
(1) quantization \cite{Jouppi:2017:IPA:3079856.3080246} and code compilation, techniques that can be applied  to any network, and 
(2) techniques which analyze the structure of the network and systematically prune connections or neurons~\cite{han2015deepcompression,Cun}. 
While the first category is useful, it has limited impact on the network size. The second category can reduce the model size much more but has several drawbacks: first, those techniques often
negatively impact model quality.  Second, they can also (surprisingly) negatively
impact inference time as they transform dense matrix operations into sparse
ones, which can be substantially slower to execute on GPUs which do not support
efficiently sparse linear algebra\cite{han2015deepcompression}. 
Third, these techniques generally start by optimizing a
particular architecture for prediction performance, and then, as a post-
processing step, applying compression  to generate a smaller model that meets the resource constraints of the deployment setting.  
Because the network architecture is essentially fixed during this post-processing,   model architectures that work better in small settings may be missed -- this is
especially true in large networks like many-layered CNNs,  where it is
infeasible to try explore even a small fraction of possible network
configurations.

%I am still not super happy with the flow. Will work on it more. 

In contrast, in this paper we present a new and surprisingly simple method to simultaneously optimize network size and model performance. The key idea is to learn the right
network size at the same time that we optimize for prediction performance. 
Our approach, called \shrink, starts with an {\it over-sized} network, and dynamically shrinks it by eliminating unimportant neurons---those that do not contribute to prediction performance---during training time. 
We achieve this, by introducing  a new layer, called \swl, which can switch neurons on and off, and is co-optimized while training the neural net.
Furthermore, the layer-based approach makes it not only easy to implement \shrink in various neural net frameworks, but also to use it as part of existing network architectures.  
%\tim{There was a lot of repetition. Tried to improve it. Old text is below }
%To do this, \shrink must detect the unimportant neurons, and remove them from the network. 
%The removal of entire neurons (as opposed to just removing connections as in \cite{Hassibi1994,Cun}) is crucial, because it leads to dense networks, which in turn leads to better inference performance.
%To support the detection and removal of unimportant neurons we introduce a new layer, called \swl. 
%Introducing this new layer simply requires us to add a term to the loss function to optimize the parameters of this layer during training. 
\shrink has two main benefits. First, it explores the architecture of models that are both small and perform well, rather than starting with a high-performing model and making it small.  
\shrink accomplishes this goal by using a single new  hyperparameter that effectively models the target network size.  
% This allows us to efficiently generate a family of smaller and
% accurate models without requiring an exhaustive and expensive hyperparameter search over
% the number of neurons in each layer.  
Second, in contrast to existing neural network compression techniques~\cite{Aghasi2016,han2015deepcompression}, our approach results in models that are not only small, but where the weight matrices are dense, leading to better inference time.

In summary, our contributions are as follows: 

%\begin{compactenum}

%\item 
\noindent\textbf{1.} We propose a novel technique based on dynamically switching
on and off neurons, which allows us to optimize the network size while the
network is trained. 

\noindent\textbf{2.} We extend our technique to remove entire neurons,
 leading not only to smaller networks, but also
dense matrices, which yield improved inference times as networks shrink. Furthermore, our switching
layers used during training can be safely removed before the model is used for
inference, meaning they add no additional overhead at inference time.

%\item \srm{xxx postprocessing techniques to strip out switch layers}

\noindent\textbf{3.} We show that our technique is a relaxation of group
LASSO~\cite{Yuan2006} and prove that our problem admits many global minima.
%\tim{I gave the main idea of the proof did you include the proof in the NN submission?}
%\gl{I gave the main idea of the proof in the manuscript and have all the full proofs in the supplementary materials. Should we rephrase this then or is good as it is ?}
%\tim{That is fine}

\noindent\textbf{4.} We evaluate \shrink with both fully-connected as well as
convolutional neural networks. For \texttt{CIFAR10}, we achieve the same
accuracy as a traditionally trained network while reducing the network size by
a factor of $2.2X$. Further, while sacrificing just $1$\% of
performance, \shrink finds networks that are 35X smaller. All in all, this
leads to speedups in inference time of up to 6X.

%\end{compactenum}

%!TEX root=paper.tex
\pdfoutput=1
\section{Related Work}

%\begin{itemize}
%  \item post-training compression techniques -- brain damage , 
%  \item group sparsity e.g., \cite{Scardapane2017} and non-parametric neural networks -- 
%  \item training dynamics paper: first overfitting and then randomization?, \gl{Here is the ref, if you can introduce it in the flow \cite{Shwartz-Ziv2017}}
%\end{itemize}

There are several lines of work related to optimizing network structure. 

%Given the importance of network structure, many researchers have explored the
%problem of finding the best network structure for a given learning task.  The
%proposed techniques broadly fall into five categories: random search andbrute
%force search, hyperparameter optimization, model compression after training,
%resizing models during training, and automated architecture search methods.

\noindent\textbf{Hyperparameter optimization techniques: }
One way to optimize network architecture is to use 
hyperparameter optimization. Although many methods have been 
proposed for hyperparameter optimization, simple techniques such as randomized
search have been shown to work suprisingly well in practice~\cite{BergstraJAMESBERGSTRA2012,Snoek12}.
%  Brute force search of network sizes is
%also become more practical due to faster and more powerful
%hardware~\cite{molchanov2016pruning}.  
Alternative more advanced techniques include Bayesian techniques and/or various bandit algorithms (e.g.~\cite{li2016hyperband, jamieson2016})
%The more complex methods include Bayesian techniques such as~\cite{Snoek12} that select hyperparameter combinations from uncertain areas of the hyperparameter space. 
%Recently, methods based on bandit algorithms (e.g. ~\cite{li2016hyperband, jamieson2016}) have also been proposed to tune hyperparameters by quickly discarding  model configurations that perform badly. 
Although these methods can be used to tune the size of each layer in a network,  in practice, related work presents limited experimental evidence regarding this, likely because treating each layer as a hyperparameter would lead to an excessively large search space.
In contrast, with \shrink, the size of the network can be tuned with 
a single parameter. 
Recently, methods based on reinforcement learning have been proposed (\cite{Zoph2017b,Zoph2017a}) and shown to generate very accurate
networks (NAS-Net). However as stated in \cite{Zoph2017b}, they still used the popular heuristic that doubles the number of channels every time the dimension of features is reduced without challenging it.

%\gl{Add another paragraph on architectures that are small by design ? (MobileNet
%ShuffleNet)} \ra{you could add these citations in the intro as motivation for small networks, but I think it's already clear. I think it's also not as relevant related work as the other paragraphs already here}

% More importantly, none of the hyperparameter optimization methods focuses on
% finding small networks, which is a crucial property of ShrinkNets, necessary to
% achieve good inference times.

%As noted before, all of the above techniques require many tens
%to hundreds of models to be trained, making this process computationally
%inefficient and slow.  More practically, the hyperparameter optimization
%literature does not evaluate their methods on network size and instead focuses
%on optimization hyperparameters such as learning rates and weight decay
%parameters.

\noindent\textbf{Model Compression: }Model compression techniques focus on
reducing the model size \emph{after} training, in contrast to \shrink, which
reduces it \emph{while} training. 
Optimal brain damage~\cite{Cun} identifies connections in a network that are
unimportant and then prunes these connections.
DeepCompression~\cite{han2015deepcompression} takes this one step further and in
addition to pruning connections, it quantizes weights to make inference
extremely efficient.  A different vein of work such as ~\cite{romero2014fitnets,
hinton2015distilling} proposes techniques for distilling a network into a
simpler network or a different model. Because these techniques work after
training, they are orthogonal and complementary to \shrink. Further,
some of these techniques, e.g.,~\cite{Han2015,Cun}, produce sparse matrices that
are not likely to improve inference times even though they reduce network size.
%Unlike our technique which works during
%training, these techiques are used after training and it would be interesting to
%apply them to ShrinkNets as well. 
%\cite{Abadi2016b} share the common goal of
%removing entire blocks of parameter to maintain dense matrices, however their
%method only applies to convolutional layers.

%\noindent\textbf{Auto-ML: } Some work focuses on automatically learning
%model architecture through the use of genetic algorithms and reinforcement
%learning techniques~\cite{DBLP:journals/corr/ZophL16, zoph2017learning}. These
%techniques are focused on learning higher-level architectures (e.g., building
%blocks for neural network architectures). In particular, they require to train
%full models and may take weeks to converge. 

%do not
%focus on finding small but well-performing networks for inference, which is the
%goal of ShrinkNets.
%\tim{Argument is not really convincing, but those techniques require to train
%full models and might take weeks to converge. }

\paragraph{Dynamically Sizing Networks}The techniques closest to our
proposed method are those based on group sparsity such as
~\cite{Scardapane2017,Alvarez2016a}, nuclear norm \cite{Alvarez2017a}, low-rank constraints \cite{Zhou2016}, exclusive sparsity \cite{Yoon}, and even
physics-inspired methods \cite{Wen2017}. In \cite{Wen2016}, authors look beyond
removing channels and experiment with shape and depth. In~\cite{Philipp}, the authors propose a method called Adaptive Radial-Angular Gradient Descent that adds and removes neurons on the
fly via an $l_2$ penalty. 
This approach requires a
new optimizer and takes longer to converge compared to \shrink.
\cite{Liu} is similar to \shrink in that they both scale each
channel/neuron by a scalar. Our approach is more general since it can be used with any architecture, does not depend on batch normalization layers, and in contrast to \cite{Liu} we propose some implementation details to make the framework more practical in section~\ref{sec:practice}.
% Finally, \cite{Scardapane2017}
%presents a method that also deactivates neurons using a loss function based on
%group-sparsity. However, the exact details of how their method works are not
%given, and their experimental results (on a small, fully connected network), are
%substantially worse than ours as shown in Section~\ref{sec:eval}.  
Most of these methods train for sparsity and deactivate neurons at the end of the training process except \cite{Alvarez2017a} that do a single step of garbage collection at epoch 15. Our pipeline allows early detection of the least important neurons/channels and take advantage from it to speed up training.
%!TEX root=paper.tex
\pdfoutput=1

\section{The Smallify Approach}
\label{sec:approach}

In this section we describe the \shrink approach. We discuss first the new \swls
which are used to deactivate neurons, followed by a description of how we adapt
the training loss function.

%In this section we describe the \shrink approach, which allows us to learn
%the network size as part of the training process. The key idea is the addition
%of \swls after each layer in the network;  these layers
%allow us to selectively deactivate neurons during training.  We start by introducing the
%basic method, and then explain how to adapt the training procedure to support this new
%layer. 

\subsection{Overview}

At a high-level, our approach consists of two interconnected stages.
The first one identifies neurons that do not improve the
prediction accuracy of the network and deactivates them. 
The second stage then {\it removes} neurons from the network (explicitly
shrinking weight matrices and updating optimizer state) thus leading to 
smaller networks and faster inference. 

\noindent\textbf{Deactivating Neurons On-The-Fly: }During the first stage,
\shrink applies an on/off switch to every neuron of an initially over-sized
network. We model the on/off switches by multiplying each input (or output) of
each layer by a parameter $\bm{\beta}\in\{0,1\}$. A value of $0$
will deactivate the neuron, while $1$ will let the signal go through. These
switches are part of a new layer, called the \swl; this layer applies 
to fully connected as well as convolutional layers.

Our objective is to minimize the number of ``on'' switches to reduce the model size as much
as possible while preserving prediction accuracy. This can be achieved by jointly
minimizing the training loss of the network and applying an $l_0$ norm to the 
$\bm{\beta}$ parameters of the \swl. 
Since minimizing the $l_0$ norm is an NP-Hard problem, we instead relax the constraint
to an $l_1$ norm by constraining $\bm{\beta}$ to be a real number
instead of a binary value.
% Because finding an optimal binary
% assignment is an NP-Hard problem, we relax the problem by allowing
%  $\bm{\beta}$ to be a real number
% instead of a binary value and constrain it using the L1 instead of L0 norm.

\noindent\textbf{Neuron Removal: } During this stage, the neurons that are
deactivated by the switch layers are actually removed from the network,
effectively shrinking the network size. This step improves inference times.
We choose to remove neurons at training time
because we have observed that this allows the remaining active neurons to adapt
to the new network architecture and we can avoid a post-training step to prune
deactivated neurons.

% Existing techniques focus on neuron removal
% after training, and require an extra fine-tuning process to compensate for the
% removal.
% \srm{cite something?} \gl{I think we are fine or we could actually
% cite any of these paper}

Next we describe in detail the switch layer as well as and the training process
for \shrink, and then describe the removal process in
Section~\ref{neuron_killing}.

\subsection{The Switch Layer}

Let $L$ be a layer in a neural network that takes an input tensor $\bm{x}$ and
produces an output tensor $y$ of shape $\left(c \times d_1 \times \dots \times
d_n\right)$ where $c$ is the number of neurons in that layer.  For instance, for
fully connected layers, $n$=0 and the output is single dimensional vector of size $c$
(ignoring batch size for now) while for a 2-D convolutional layer, $n$=2 and $c$
is the number of output channels or feature maps.

We want to tune the size of $L$ by applying a \swl, $S$, containing $c$
switches.  The \swl is parametrized by a vector $\bm{\beta} \in \mathbb{R}^c$ such that the result
of applying $S$ to $L(\bm{x})$ is a also a tensor size $\left(c \times d_1
\times \dots \times d_n\right)$ such that: 
%Suppose we wish to tune the size of $L$ by applying a switch layer.
%A switch layer $S$ applied to the output of $L$ can be parametrized by a 
%vector 
\begin{equation} 
S_{\beta}(L(\bm{x}))_{i,...} = \bm{\beta}_iL(\bm{x})_{i, ...} \forall i \in [1\ldots c]
\end{equation}
Once passed through the switch layer, each output channel $i$
produced by $L$ is scaled by the corresponding $\bm{\beta}_i$. Note that when
$\bm{\beta}_i = 0$, the $i^{\text{th}}$ channel is multiplied by zero and will not
contribute to any computation after the switch layer. If this happens, we say
the switch layer has {\it deactivated} the neuron corresponding to channel $i$ of layer $L$.

We place \swl after each layer whose size we wish to tune; these are
typically fully connected and convolutional layers. We discuss next how to train
\shrink.

\subsection{Training \shrink} 

%To tune the size of a network, we place Switch Layers after each layer whose size
%we wish to tune; these are typically the fully connected and convolutional layers
%in a network.
%Since the switch layers are adept at deactivating neurons, we start with an
%oversized network (i.e. network with more capacity than required) and then use
%switch layers to deactivate or kill off neurons that are unnecessary.

%Formally, 
For training, we need to account for the effect of the \swls on the loss
function. The effect of \swls can be expressed in terms of a sparsity constraint
that pushes values in the $\bm{\beta}$ vector to 0. In this way, given a neural network
parameterized by weights $\bm{\theta}$ and switch layer parameters $\bm{\theta}$, we
optimize \shrink loss as: 
\begin{equation}
\small
  L_{SN}(\bm{x},\bm{y};\bm{\theta}, \bm{\beta}) = L(\bm{x}, \bm{y}; \bm{\beta}) +
  \lambda\norm{\beta}_1 + \lambda_2\norm{\bm{\theta}}_p^p
\end{equation}
This expression augments the regular training loss with a regularization
term for the switch parameters and another on the network weights.

Interestingly, there exists a connection between \shrink and group sparsity regularization (LASSO) which we will discuss in the following subsection.

\pdfoutput=1
%\subsection{Relation to Group Sparsity}

\subsection{Relation to Group Sparsity (LASSO)} 

%\noindent\textbf{Relation to Group Sparsity (LASSO): } 

\shrink removes neurons,
i.e., inputs/outputs of layers. For a fully connected layer defined as:
\begin{equation} \label{fully_connected}
  f_{\bm{A}, \bm{b}}(\bm{x})=a(\bm{Ax + b})
\end{equation}
where $\bm{A}$ represents the connections and $\bm{b}$ the bias,
removing an input neuron $j$ is equivalent to having $\left(\bm{A}^T\right)_j =
\bm{0}$. Removing an output neuron $i$ is the same as setting $\bm{A}_i = \bm{0}$
and $\bm{b}_i = 0$. Solving optimization problems while trying to set entire
group of parameters to zero is the goal of group sparsity regularization
\cite{Scardapane2017}. 
In  any partitioning of the set of parameters $\bm{\theta}$ defining a model in $p$
groups: $\bm{\theta} = \bigcup_{i=1}^P \bm{\theta}_i$, group sparsity 
penalty is defined as: 
\vspace{-0.1in}
\begin{equation}
    \label{full_def}
  \Omega_\lambda^{gp} = \lambda \sum_{i=1}^p \sqrt{\mathbf{card}(\bm{\theta_i}}) \norm{\bm{\theta_i}}_2 \\
\end{equation}
\vspace{-0.2in}
% \srm{define notation -- what is $\lambda$;  what is $\#\theta$, what is $\Omega$? Where does the square root come from?} \gl{This is how it is defined in the
% original paper, it would require a few sentences to explain that, do really
% need to do ?}

with $\lambda$ being the regularization parameter.
In fully-connected layers, the groups are either columns of
$\bm{A}$ if we want to remove inputs, or rows of $\bm{A}$ and the corresponding
entry in $\bm{b}$ if we want to remove outputs. For simplicity, we focus
our analysis on the simple one-layer case. As filtering outputs does
not make sense in this case,  we only consider removing inputs. The
group sparsity regularization then becomes (when $\sqrt{n}$ is folded into the $\lambda$)
\vspace{-0.1in}
\begin{equation} \label{group_sparsity_regularization}
  \Omega_\lambda^{gp} = \lambda \sum_{j=1}^p \norm{\bm{\left(A^T\right)_j}}_2 \\
\end{equation}
\vspace{-0.2in}

% Because $\forall i, \#\bm{\theta}_i = n$, embedded $\sqrt{n}$ inside $\lambda$.
%  \ra{is \# the general way of expressing
% cardinality? why not $|x|$?,}  \gl{In europe |x| is for the abolute value, and I was using it in the notation section that someone removed, therefore there would have been a notation conflict}to make the notation simpler, we

Interestingly, group sparsity and \shrink try to achieve the same goal and are closely related. First let's recall the two problems. In the
context of approximating $\bm{y}$ with a linear regression from features
$\bm{x}$, the two problems are:%
\vspace{-0.25in}
\begin{multicols}{2}
    \begin{equation*}
        \text{\shrink: } \min_{\bm{A}, \bm{\beta}} \norm{\bm{y} - \bm{A}\diag{\bm{\beta}}\bm{x}}_2^2 + \lambda \norm{\bm{\beta}}_1
    \end{equation*}
    \break
    \begin{equation*}
        \text{\textbf{Group sparsity: }}\min_{\bm{A}} \norm{\bm{y} - \bm{A}\bm{x}}_2^2 + \Omega_\lambda^{gp}
    \end{equation*}
\end{multicols}
\vspace{-0.1in}

We can prove that under the condition: $\forall j\in \intint{1, p},
\norm{\left(\bm{A}^T\right)_j}_2 = 1$ the two problems are equivalent by taking
$\bm{\beta}_j = \norm{\left(\bm{A}^T\right)_j}_2^2$, and replacing $\bm{A}$ by
$\bm{A}\left(\diag{\bm{\beta}}\right)^{-1}$. However, if we relax this constraint
then \shrink becomes non-convex and has no global minimum. The latter is true
because one can divide $\bm{\beta}$ by an arbitrarly large constant and
multipliying $A$ by the same value. Fortunately, by adding an extra term to the \shrink
regularization term we can avoid that problem and prove that:
\begin{equation}
  \min_{\bm{A}, \bm{\beta}} \norm{\bm{y} - \bm{A}\diag{\bm{\beta}}\bm{x}}_2^2 + \Omega_\lambda^s + \lambda_2\norm{A}_p^p
\end{equation}
has global minimums for all $p>0$. More specifically there are at least $2^k$,
where $k$ is the total number of components in $\bm{\beta}$. Indeed, for any
solution, one can obtain the same output by flipping any sign in $\bm{\beta}$
and the corresponding entries in $\bm{A}$.  This is the reason we defined the
regularized \shrink penalty above in \cref{full_def}.
In practice, we observed that $p=2$ or $p=1$ are good a choice; note that the latter
will also introduce additional sparsity into the parameters because the $l_1$ is, 
thest best convex approximation of the $l_0$ norm.

%!TEX root=paper.tex
\pdfoutput=1

\section{Smallify in Practice}
\label{sec:practice}

In this section we discuss practical aspects of \shrink{}, including 
neuron removal  and several optimizations.

%\paragraph{On-The-Fly Neuron Removal}
\noindent\textbf{On-The-Fly Neuron Removal.}
\label{neuron_killing}
Switch layers are initialized with weights sampled from $\mathcal{N}(0,1)$;
their values change as part of the training process so as to switch \emph{on}
or \emph{off} neurons. Using gradient descent, it is very unlikely that the unimportant
components of $\bm{\beta}$ will ever be exactly $0$. In most cases, irrelevant neurons will see their \swl oscillate close to 0, while never reaching 0, influenced solely by the $L_1$ penalty. Our goal is
to detect this situation and effectively force them to $0$ to deactivate them. We
evaluated multiple screening strategies but the most efficient and flexible one
was the \textbf{Sign variance strategy}: At each update we measure the sign of
each component of $\bm{\beta}$ ($-1$ or $1$). We maintain two metrics: the
exponential moving average (EMA) of its mean and variance. 
When the variance exceeds a predefined threshold, we assume that the neuron does not contribute significantly to the output, so we effectively deactivate it. 
This strategy is parametrized
by two hyper-parameters, the threshold but also the momentum of the statistics we keep.

\noindent\textbf{Preparing for Inference.} With \shrink we obtain
reduced-sized networks during training, which is the first steps towards faster
inference. This networks are readily available for inference. However, because
they include switch layers---and therefore more parameters---they introduce
unnecessary overhead at inference time. To avoid this overhead, we reduce the
network parameters by combining each switch layer with its respective network
layer by multiplying the respective parameters before emitting the final trained
network. As a result, the final network is a dense network without any switching layers. 

\noindent\textbf{Neural Garbage Collection.} \shrink decides on-the-fly which
neurons to deactivate. Since \shrink deactivate a large fraction of neurons, we must
dynamically remove these neurons at runtime to not unnecessarily impact network
training time. We implemented a neural garbage collection method as part of our
library which takes care of updating the necessary network layers as well as
updating optimizer state to reflect the neuron removal.

%!TEX root=paper.tex
\pdfoutput=1
\section{Evaluation}
\label{sec:eval}

The goal of our evaluation is to explore (1) whether, by varying $\lambda$,
\shrink can efficiently explore (in terms of number of training runs)  the
spectrum of high-accuracy models from small to large, on both CNNs and fully
connected networks.  Our results show that, for each network size, we obtain
models that perform as well or better than \textit{Static Networks}, trained via
traditional hyperparameter optimization;  (2) whether, because these  smaller
networks are dense, they result in improved inference times on both CPUs and
GPUs; and (3) whether the \shrink approach results in network architectures that
are substantially different than the best network architectures (in terms of
relative number of neurons per layer) identified in the literature.

We implemented \swls and
the associated training procedure as a library in
pytorch~\cite{paszke2017automatic}. The layer can be freely mixed with other
popular layers such as convolutional layers, batchnorm layers, fully connected
layers, and used with all the traditional optimizers. We use our implementation
to evaluate \shrink throughout the evaluation section.

\subsection{Can \shrink achieve good accuracy?}
%\subsection{Performance vs. Traditional Methods}

To answer this question we compare \shrink with a traditional network. In both
cases, we need to perform hyperparameter optimization to explore different
network sizes. We perform random search, which is an effective technique
for this purpose \cite{BergstraJAMESBERGSTRA2012}. We evaluate \shrink on two
architectures. One for which it is not possible to explore the entire
space of network architectures (\texttt{VGG}) and one for which it is
possible to do so (3 layers perceptron).

We assume no prior knowledge on the optimal batch
size, learning rate, $\lambda$ or weight decay ($\lambda_2$). Instead, we
trained a number of models, randomly and independently selecting the values of
these parameters from a range of values commonly used in practice. Training is done using the
\textit{Adam} optimizer \cite{Kingma2015a}. We start with
randomly sampled learning rate; we divide the learning rate by $10$ every $5$
consecutive epochs without improvement. We stop when the learning rate is under
$10^{-7}$.  We pick the epoch with the best validation accuracy after the size
of network converged and report the corresponding testing accuracy. We also
measure the total size, in terms of number of floating point parameters,
excluding the \swls because as described in \cref{neuron_killing}, these are
eliminated after training.

\subsubsection{Large Network Setting: \texttt{CIFAR10}}

\texttt{CIFAR10} is an image classification dataset containing $60000$ color
images $(3 \times 32 \times 32)$, belonging to $10$ different classes. We use it
with the \texttt{VGG16} network \cite{Srivastava2014}. We applied \shrink to the
VGG16 network by adding \swls after each \textit{BatchNorm}  and each fully
connected layer (except for the last layer). Recall that \shrink assume that the starting
size of the network is an upper bound on the optimal size. Thus, we started with
a network with 2x the original size for each layer.

As the baseline we use a fixed-sized network, which architecture is configured by a total of 13~parameters for the
convolutional layers and $2$ for the fully connected layers. \shrink
effectively fuse all these parameters in a single $\lambda$. 
However, for traditional conventional architectures where all of these parameters are free, it is infeasible to obtain a reasonable sample for such a large search space. 
To obtain a baseline, we therefore use the same conventional heuristic that the original VGG architecture and many other CNNs use, which doubles the number of channels   after every \textit{MaxPool} layer.
For \textit{Static Networks} we sample the size between $0.1$ and $2$ times the size
original one, designed for ImageNet. We report the same numbers as we did for
\shrink and we compare the two distributions.

The results are shown in the top figure of \cref{figure_CIFAR10}, with blue dots
indicating models produced by \shrink and orange dots indicating static
networks.  model, we plot its accuracy and model size. The lines show the Pareto
frontier of models in each of the two optimization settings. \shrink explore the
trade-off between model size and accuracy more effectively. Note that the best
performing \shrink model has $92.07\%$ accuracy which is identical to the
accuracy of the static network, while the \shrink model is $2.22$ times smaller.
In addition, if we give up just 1\% error, \shrink find a model that is 35.5
times smaller than any static network that performs as good.

\subsubsection{Small Network Setting: \texttt{COVERTYPE}}

The \texttt{COVERTYPE} \cite{Blackard:1998:CNN:928509} dataset contains $581012$
descriptions of geographical area (elevation, inclination, etc...) and the goal
is to predict the type of forest growing in each area. We picked this dataset
for two reasons. First it is simple, such that we can reach good accuracy with
only a few fully-connected layers. This is important because we want to show
that \shrink find sizes as good as \textit{Static Networks}, even if
we are sampling the entire space of possible network sizes. Second, Scardapane
et al~\cite{Scardapane2017} perform their evaluation on this dataset, which
allows us to compare the results obtained by our method with the method in
~\cite{Scardapane2017}. We compare \shrink against the same architecture
used in \cite{Scardapane2017}, i.e., a three fully-connected layers network with no
\textit{Dropout} \cite{Srivastava2014} and no \textit{BatchNorm}. 
%\gl{Should we say
%here that we don't expect Dropout to work here ? I could write an entire
%paragraph about it if needed}. 
In this case, for the \textit{Static Networks}, we independently sample the
sizes of the three different layers to explore all possible architectures. 

The results are shown in the top figure of \cref{figure_COVER}.  Here, {\it Static}
method finds models that perform well at a variety of sizes, because it is able
to explore the entire parameter space.  This is as expected;  the fact that
\shrink perform as well as the Static indicates that \shrink are doing an
effective job of exploring the parameter space using just the single $\lambda$
parameter.  Note that the best performing \shrink models has $96.91\%$ accuracy
while the best static model is only $96.66\%$ accurate, while the \shrink shrink
model is $2.51$ times smaller. In addition, if we give up just 0.5\% error,
\shrink find a model that is 38.6X smaller than any static network with
equivalent accuracy.

% \subsubsection{Summary}
% 
% We  demonstrated that it is possible to achieve networks with good accuracy
% when using \shrink both when the network space cannot be explored entirely
% (\texttt{CIFAR10}) and when it can, e.g., \texttt{COVERTYPE}. The most important
% result is not that \shrink find networks of good accuracy, but that those
% networks are much smaller than those found by a static method. The impact of the
% network size on inference time is the subject of our next evaluation goal.

\begin{figure}
\centering
\begin{minipage}{.49\textwidth}
  \centering
  \includegraphics[width=\textwidth]{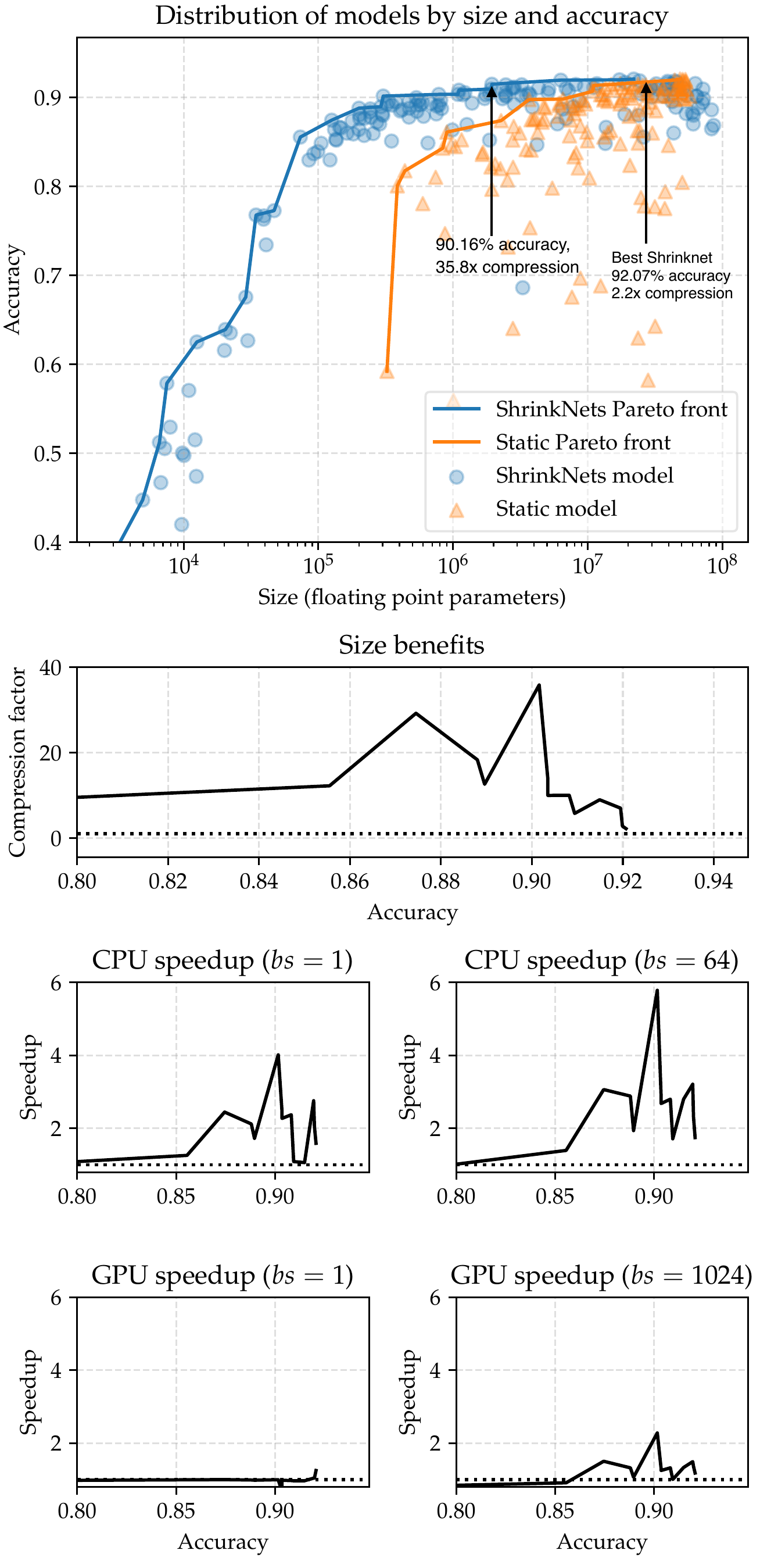}
  \captionof{figure}{
\label{figure_CIFAR10} Summary of the result of random
search over the hyper-parameters the \texttt{CIFAR10} dataset}
  \label{fig:test1}
\end{minipage}%
\hspace*{0.02\textwidth}\begin{minipage}{.49\textwidth}
  \centering
  \includegraphics[width=\textwidth]{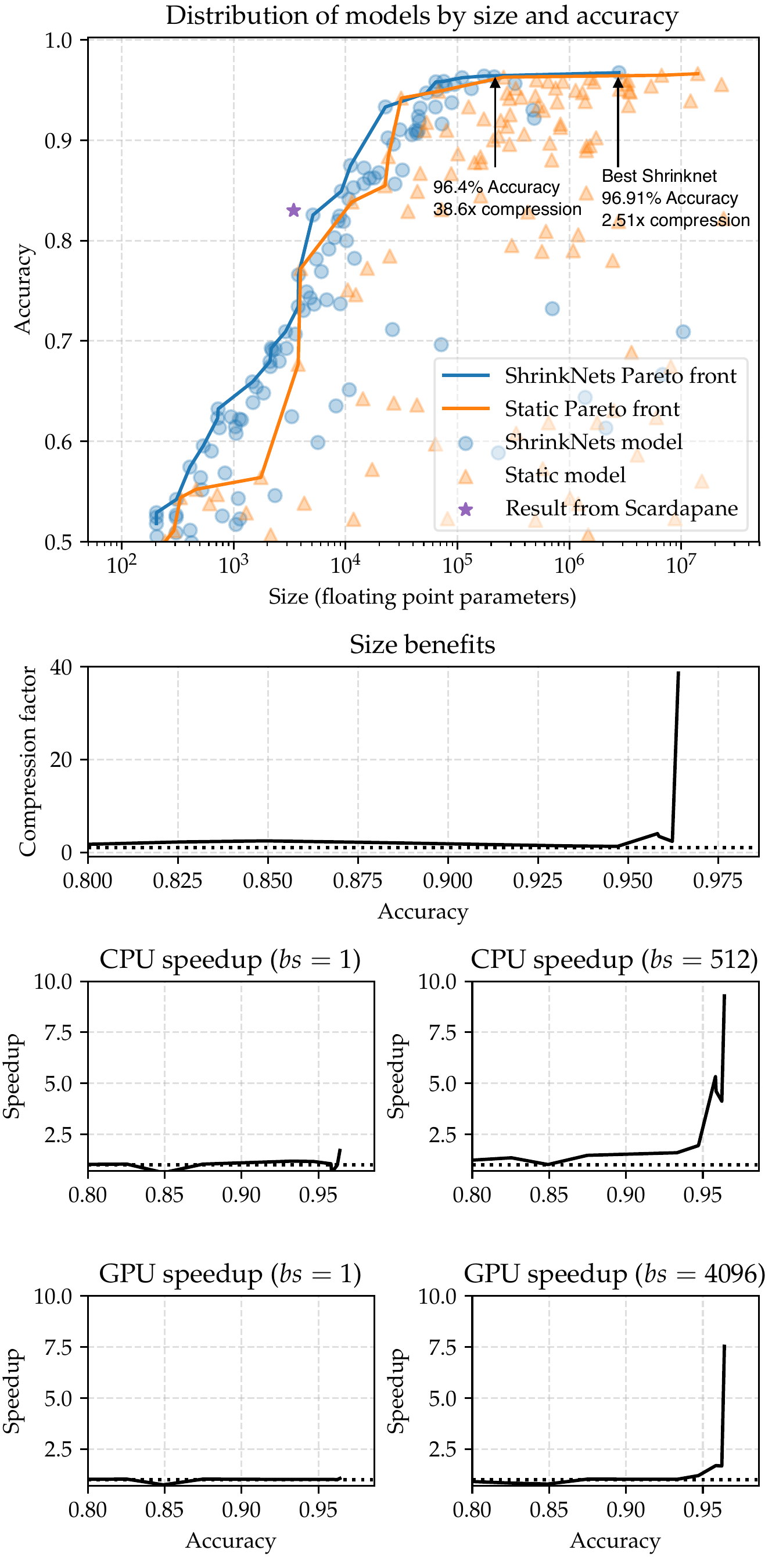}
  \captionof{figure}{
      \label{figure_COVER} Summary of the result of random
search over the hyper-parameters the \texttt{COVERTYPE} dataset
}
\end{minipage}
\end{figure}

\subsection{Can Smallify speed up inference?}

The previous experiment showed that \shrink find networks of similar or better
accuracy than static networks that are much smaller. As noted in the
introduction, for some applications, compact models that offer fast inference
times are as important as absolute accuracy.
%This observation motivates our 
% the experiment approach described in the previous section 
% and shown in the top of Figure~\ref{figure_CIFAR10} 
% and~\ref{figure_COVER}:
% for a  desired target accuracy, the Pareto optimal line shows the smallest
% network
% that satisfies achieves a given accuracy. 
In this section, we study the relationship between accuracy, network size and
inference time.  To do this, we select the smallest model that achieves a given
accuracy for the both \shrink and Static approach.  For each model, we measure
the time to run inference with the model.  We then compute the ratio of the
network size and inference time between \shrink and Static at each accuracy
level, and plot them on the bottom of Figure~\ref{figure_CIFAR10}
and~\ref{figure_COVER}.  We limit our plots to the models with $80-100\%$
accuracy range because those are the ones that we consider to be practically
useful.

The middle plot in each figure shows the ratio of model size between \shrink
and Static (values $>$1 mean \shrink are smaller) at different accuracy levels.
These figures show that is that size improvements are are particularly
significant for  \texttt{CIFAR10}. In the range of accuracies we are interested
in, improvements in size go from 4x to 40x. The fact
that the  \texttt{COVERTYPE} networks are not dramatically smaller is expected:
as the distribution at the top of Figure~\ref{figure_COVER} shows, the static
method is able to explore most of the parameter search space.

For speedup, we experimented with both CPUs and GPUs. For each data set/GPU/CPU
combination, we show results with batch size 1, as well as with a batch size
large enough to fully utilize the hardware on each dataset and hardware
configuration. Note that when using a batch size of $1$ on GPU, we do not expect
to (and do not) observe any improvement because inference times are very small
(typically about 10 $\mu s$), such that setup time dominates overall runtime.

The bottom four graphs in each figure show the results.  Again, the {\tt
CIFAR10} results show the benefit of the \shrink approach most dramatically.  On
CPU, speedups range up to 6x depending on the batch size, with many models
exceeding 3x speedup. In general, speedups are less than compression ratios, due
to overheads in problem setup, invocation, and result generation  in
Python/PyTorch.  On GPU, the speedups are less substantial because the CUDA
benchmarking utility that we use for evaluation can choose better algorithms for
larger matrices which masks some of our benefit, although they are still often
1.5x--2x faster for large batch sizes.

A key takeaway of these speedup results is that, unlike local sparsity
compression methods, our methods' improvement on size translates directly to
higher throughput at inference time~\cite{Han2015}.

\subsection{Architectures obtained after convergence}

\shrink effectively explore the frontier of model size and accuracy. For a
given target accuracy, the size needed is significantly smaller than when we use the
"channel doubling" heuristic commonly used to size convolutional neural networks.
This suggests that this conventional heuristic may not in fact be optimal,
especially when looking for smaller models.  Empirically we observed this to
often be the case.  For example, during our experimentations on the
\texttt{MNIST} \cite{Lecun1998} and \texttt{FashionMNIST} \cite{Xiao2017}
datasets (not reported here due to space constraints), we observed that even
though these datasets have the same number of classes, input features, and
output distributions, for a fixed $\lambda$ \shrink converged to
considerably bigger networks in the case of \texttt{FashionMNIST}. This evidence
shows that optimal architecture not only depends on the output distribution or
shape of the data but actually reflects the dataset.  This makes sense, as
\texttt{MNIST} is a much easier problem than \texttt{FashionMNIST}.

To illustrate this point on a larger dataset, we show two examples of
architectures learned by \shrink in Figure~\ref{fig:network_size_evolution}. 
In the plot, the dashed line
shows the number of neurons in each layer of the original VGG net, and the
shaded regions show the size of the \shrink as it converges (with the darkest
region representing the fully converged network).  Observe that the final
network that is trained looks quite different in the two cases, with the optimal
performing network appearing similar to the original VGG net, whereas the
shrunken network allocates many fewer neurons to the middle layers, and then
additional neurons to the final fewer layers.

\begin{figure}[htb]
\begin{center}
\vspace{-.1in}
\includegraphics[width=1\columnwidth]{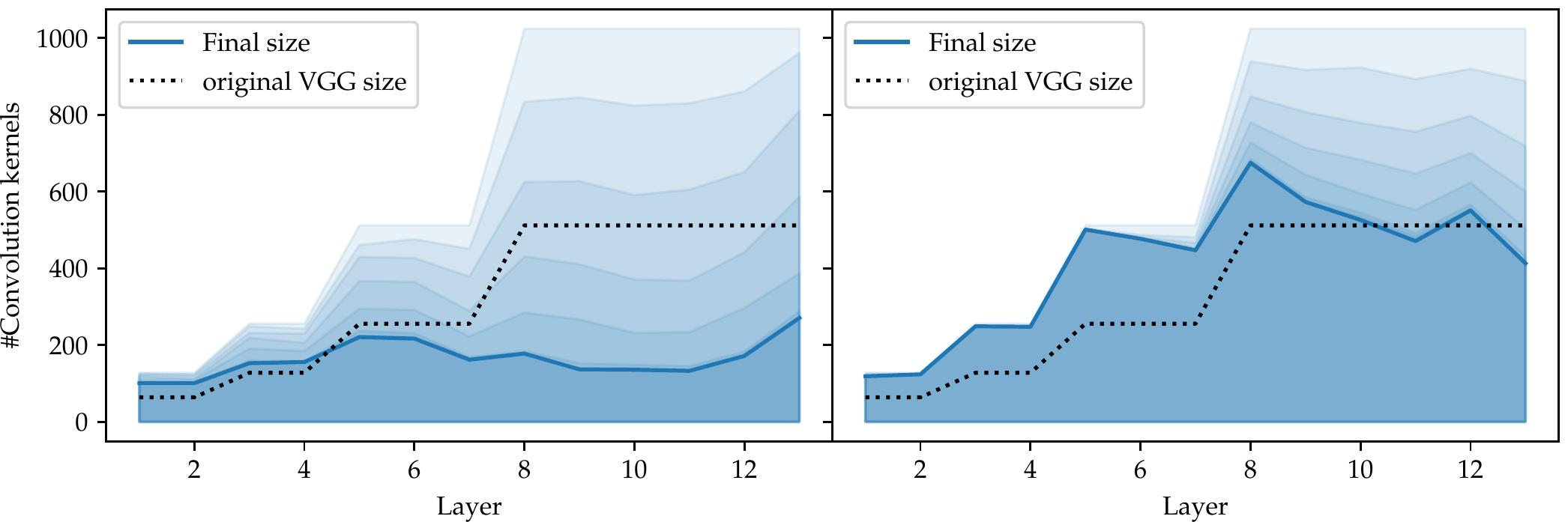}
\vspace*{-5mm} 
\caption{ Evolution of the size of
  each layer over time (lighter: beginning, darker: end). On the left a simpler model with $90.5\%$
  accuracy, on the right a very large
  network performing $92.07\%$.
} 
\label{fig:network_size_evolution}
\end{center}
\vspace*{-4mm}
\end{figure}
%!TEX root=paper.tex
\pdfoutput=1
\section{Conclusion}

We presented \shrink, an approach to learn deep network sizes while training.
\shrink employs a \swl, which deactivates neurons, as well as
of a method to remove them, which reduces network sizes, leading to faster
inference times. We demonstrated these claims on on two well-known datasets, on
which we achieved networks of the same accuracy as traditional neural
networks, but up to 35X smaller, with inference speedups of up
to 6X.

\FloatBarrier
%!TEX root=paper.tex
\pdfoutput=1
\section{Appendix}

%\begin{figure}[t]
%\begin{center}
%\includegraphics[width=.7\columnwidth]{regressions}
%\vspace*{-5mm}
%\caption{\label{sparsity_accuracy}Loss/Sparsity trade off comparison between Group Sparsity and Shrinknets on linear and logistic regression. From top to bottom and left to right we show the results for \texttt{scm1d}, \texttt{oes97}, \texttt{gina\_prior2} and \texttt{gsadd}.}
%
%\end{center}
%\vspace*{-4mm}
%\end{figure}

%Unless stated explicitly, all the propositions consider the Multi-Target linear regression problem.
\begin{proposition}
\label{gps_equivalence}
  $\forall (n, p) \in \mathbb{N}_+^2, \bm{y} \in \mathbb{R}^{n}, \bm{x} \in \mathbb{R}^{p} \lambda \in \mathbb{R}$:
\vspace{-.05in}
  \begin{equation*}
    \min_{\bf{A}} \norm{\bf{y} - \bf{A}\bf{x}}_2^2 + \lambda \sum_{j=1}^p \norm{\left(A^T\right)_j}_2 \\
     = \begin{cases}
         \min_{\bf{A'}, \bm{\beta}} \norm{\bm{y} - \bm{A'}\diag{\bm{\beta}}\bf{x}}_2^2 + \lambda \norm{\bm{\beta}}_1 \\
          \text{s.t.} \forall j, 1 \leq j \leq p, \norm{\left(A'^T\right)_j}_2^2 = 1
        \end{cases}
  \end{equation*}
\end{proposition}
\vspace{-.2in}
\begin{proof}
  First, we prove that there is at least one global minimum. Then, we 
    how to construct $2^k$ distinct solutions from a single global
  minimum.
  In order to prove this second statement, we first show that for any solution $\bm{A}$ to the first problem, there exists a solution in the second with the exact same value, and vice-versa.
  \vspace{-0.15in}
  \paragraph{Part 1} Assume we have a potential solution $\bm{A}$ for the first problem.  We define $\bm{\beta}$ such that $\bm{\beta}_j = \norm{\left(\bm{A}^T\right)_j}_2^2$, and $\bm{A}' = \bm{A}\left(\diag{\bm{\beta}}\right)^{-1}$. It is easy to see that the constraint on $\bm{A}'$ is satisfied by construction. Now:
  \vspace{-0.10in}
  \begin{equation*}
    \begin{aligned}
     \norm{\bf{y} - \bf{A}\bf{x}}_2^2 + \lambda \sum_{j=1}^p \norm{\left(A^T\right)_j}_2 
    = \norm{\bf{y} - \bf{A'}\diag{\bm{\beta}}\bf{x}}_2^2 + \lambda \sum_{j=1}^p \norm{\left(A'^T\right)_j\beta_j}_2 \\
     = \norm{\bf{y} - \bf{A'}\diag{\bm{\beta}}\bf{x}}_2^2 + \lambda \sum_{j=1}^p \abs{\beta_j} \cdot 1
    = \norm{\bf{y} - \bf{A'}\diag{\bm{\beta}}\bf{x}}_2^2 + \lambda \norm{\bm{\beta}}_1
\end{aligned}
  \end{equation*}
  \vspace{-0.25in}
  \paragraph{Part 2} Assuming we take an $\bm{A}'$ that satisfies the constraint and a $\bm{\beta}$, we can define $\bm{A} = \bm{A'}\diag{\bm{\beta}}$. We can apply the same operations in reverse order and obtain an instance of the first problem with the same value.
  \vspace{-0.15in}
  \paragraph{Conclusion} There is no way these two problems have different minima,
  because we are able to construct a solution to a problem from the solution
  of the other while preserving the value of the objective.
\end{proof}

\begin{proposition} \label{unconstrained_non_convex}
$     \norm{\bf{y} - \bf{A}\diag{\bm{\beta}}\bf{x}}_2^2
$ is not convex in $\bm{A}$ and $\bm{\beta}$.
\begin{proof}
  To prove this we will take the simplest instance of the problem: where everything is a scalar. We have $f(a, \beta) = \left(y - a\beta x\right)^2$. For simplicty we's take $y = 0$ and $x > 0$. If we consider two candidates $s_1 = (0, 2)$ and $s_2 = (2, 0)$, we have $f(s_1) = f(s_2) = 0$. However $f(\frac{2}{2}, \frac{2}{2}) = x > \frac{1}{2} f(0, 2) + \frac{1}{2}f(2, 0)$, which break the convexity property. Since we showed that a particular case of the problem is non-convex then necessarily the general case cannot be convex.
\end{proof}
\end{proposition}

\begin{proposition}
\label{unconstrained_shrinknet_no_min}
$     \min_{\bf{A}, \bm{\beta}} \norm{\bf{y} - \bf{A}\diag{\bm{\beta}}\bf{x}}_2^2 + \lambda \norm{\bm{\beta}}_1
$
has no solution if $\lambda > 0$.
\end{proposition}
\begin{proof}
  Let's assume this problem has a minimum $\bm{A}^*, \bm{\beta}^*$. Let's consider $2\bm{A}^*, \frac{1}{2}\bm{\beta}^*$. Trivially the first component of the sum is identical for the two solutions, however $\lambda\norm{\frac{1}{2}\bm{\beta}} < \lambda\norm{\bm{\beta}}$. Therefore $\bm{A}^*, \bm{\beta}^*$ cannot be the minimum. We conclude that this problem has no solution.
\end{proof}
\begin{proposition}
  \label{shrinknet_regularized_minimum}
For this proposition we will not restrict ourselves to single layer but the composition of an an arbitrary large ($n$) layers as defined individually as $f_{\bm{A}_i, \bm{\beta}_i, \bm{b}_i}(x) = a(\bm{A_i}\diag{\bm{\beta_i}}\bm{x} + \bm{b_i})$. 
Suppose the entire network is denoted by the function $N(\bm{x})$.
% The entire network follows as: $N(\bm{x}) = \left(\bigcirc_{i=1}^n f_{\bm{A_i}, \bm{\beta_i}, \bm{b_i}}\right)(\bm{x})$. 
For $\lambda > 0$, $\lambda_2 > 0$ and $p > 0$ we have that $\min \norm{\bm{y} - N(\bm{x})}_2^2 + \Omega_{\lambda, \lambda_2, p}^{rs}$ has at least $2^k$ global minimum where $k = \sum_{i=1}^n {\bf card}(\bm{\beta_i})$
\end{proposition}

\begin{proof}
  We  split this proof into two parts. First we show that there is at least
  one global minimum, then we will show how to construct $2^n - 1$ other distinct
  solutions with the same objective.
\\ \textbf{Part 1:}
The two components of the expression are always positive so we know that this problem is bounded by below by $0$. $\Omega_{\lambda, \lambda_2, p}^{rs}$ is trivially coercive. Since we have a sum of terms, all bounded by below by $0$ and one of them is coercive, so the entire function admits at least one global minimum.
\\ \textbf{Part 2:} Let's consider one global minimum. For each component $k$ of $\bm{\beta_i}$ for some $i$. Negating it and negating the $k^{th}$ column of $\bm{A_i}$ does not change the the first part of the objective because the two factors cancel each other. The two norms do not change either because by definition the norm is independent of the sign. As a result these two sets of parameters have the same value and by extension also a global minimum. It is easy to see that going from this global minimum, we can decide to negate or not each element in each $\bm{\beta_i}$. We have a binary choice for each parameter, there are $k = \sum_{i=1}^n {\bf card}(\bm{\beta_i})$ parameters, so we have at least $2^k$ global minima.

\end{proof}

%\subsection{Multi-Target Linear and Multi-Class Logistic regressions experiments}
%As we showed, Group sparsity share similarities with our method, and we claim
%that ShrinkNets are a relaxation of Group sparsity.  In this experiment we want
%to compare the two approaches.  We decided to focus on multi-target linear
%regression because in the single target case, groups in the Group Sparsity
%problem would have a size of one ($\bm{A}$ would be a vector in this case).
%
%The evaluation will be done on two datasets \texttt{scm1d} and \texttt{oes97}
%\cite{Spyromitros-Xioufis2016} for linear regressions and we will use \texttt{gina\_prior2} \cite{4371065} and
%the \textit{Gas Sensor Array Drift Dataset} \cite{VERGARA2012320} (that we shorten in
%\texttt{gsadd}) for logistic regressions.
%
%For each dataset we fit with different regularization parameters and measure
%the error and sparsity obtained after convergence. In this context we define
%sparsity as the ratio of columns that have all their weights under $10^{-3}$ in
%absolute value. Regularization parameters were chosen in order to obtain the
%widest sparsity spectrum. Loss is normalized depending on the problem to be in
%the $[0, 1]$ range. We summarized the results in \cref{sparsity_accuracy}. From
%our experiments it is clear that ShrinkNets can fit the data closer than Group
%Sparsity for the same amount of sparsity. The fact that we are able to reach
%very low loss demonstrate that even if our objective function is non convex, in
%practice it works as good or better as convex alternatives.

\bibliographystyle{unsrt}
\bibliography{custom,general}
\end{document}